\documentclass{article}

\usepackage{arxiv}

\usepackage[utf8]{inputenc} 
\usepackage[T1]{fontenc}    
\usepackage{hyperref}       
\usepackage{url}            
\usepackage{booktabs}       
\usepackage{amsfonts}       
\usepackage{nicefrac}       
\usepackage{microtype}      
\usepackage{lipsum}
\usepackage{graphicx}
\usepackage{mathtools}
\usepackage{bigints}
\usepackage{amssymb, amsmath, amsthm}

\graphicspath{ {./images/} }

\newtheorem{theorem}{Theorem}[section]

\title{Improvements in Computation and Usage of Joint CDFs for the N-Dimensional Order Statistic}

\date{January 29, 2019}

\author{
 Arvind Thiagarajan \\
  \texttt{arvindthiagarajan@gmail.com} \\
}

\begin{document}
\maketitle
\begin{abstract}
Order statistics provide an intuition for combining multiple lists of scores over a common index set. This intuition is particularly valuable when the lists to be combined cannot be directly compared in a sensible way. We describe here the advantages of a new method for using joint CDFs of such order statistics to combine score lists. We also present, with proof, a new algorithm for computing such joint CDF values, with runtime linear in the size of the combined list.
\end{abstract}

\section{Introduction}
Whether building attribute predictors or taking measurements over a fixed space, combining multiple lists of predictions or measurements is a common final step. When these predictions or measurements are generated from identical processes, differing from each other only as a consequence of intrinsic noise, there will likely be principled and case-specific ways to combine these values across lists. This familiar scenario corresponds to independent, identically distributed measurements of a quantity of interest.

However, it is not uncommon to handle lists that are fundamentally different from one another. In such a scenario, each list corresponds to a different proxy, i.e. a value that is monotonically increasing with a common target quantity but otherwise following a functional relationship that we do not understand. When the problem of combining such lists is supervised, this corresponds to any situation in which we may employ boosting. When the problem is unsupervised, however, we can no longer extract information from the spacing between different scores within a list, and so the only meaningful information provided by each list is a ranking of elements. Given this, order statistics would seem to play a natural role.

\section{Definition}

We begin by defining the joint CDF value for an $n$-dimensional order statistic. Given a sorted (ascending) list $R = [r_1, r_2, \cdots, r_n]$, and fixing $s_0$ to be $0$, the joint CDF value of $R$ is defined to be $V(R) = V_1(R)$, where

\begin{equation}
\label{eq:1}
    V_i(R) = \begin{cases}
                \bigints_{s_{i-1}}^{r_i}{ds_i} & \text{if $i=n$} \\
                \bigints_{s_{i-1}}^{r_i}{V_{i+1}(R) ds_i} & \text{if $0 < i < n$}
            \end{cases}
\end{equation}

or, more intuitively,

\begin{equation}
    V(R) = \int_{0}^{r_1}\int_{s_1}^{r_2}\cdots \int_{s_{n-1}}^{r_n}{ds_n ds_{n-1}\cdots ds_1} \label{eq:2}
\end{equation}

In addition, we would like to clarify terminology: throughout this paper, ranked lists of elements should be assumed sorted in descending order of score, with "larger rank" referring to smaller scoring elements appearing later in the list.

\section{Previous Usage}

We now describe how the joint CDF value was used in prior work \cite{aerts} \cite{stuart} to combine multiple lists of scores. Let us assume we begin with several ranked lists of elements. We choose as our null hypothesis that each ranked list is generated by randomly permuting a set of elements.

Define the \textbf{rank ratio} of an element in a list to be the rank of the element divided by the length of the list. For a particular element $e$, then, let $R_e$ be the sorted (ascending) list of rank ratios, computed over all the lists in which $e$ is present. \href{https://en.wikipedia.org/wiki/Majorization}{Majorization} induces a partial ordering over the space of such sorted lists as follows: for two such sorted lists $S = [s_i]$ and $R = [r_i]$ of rank ratios, we define $S \leq R$ to mean that $S$ is majorized by $R$, i.e that $s_i \leq r_i\text{ }\forall\text{ }i$. This aligns well with our intuition - if all the rank ratios for an element $e$ are smaller than the corresponding rank ratios for an element $f$, then we can unambiguously state that $e$ should come before $f$ in any merged ranking.

Using this partial ordering, we have that the p-value for a list $R$ under the chosen null hypothesis would be

\begin{equation}
    n! \int_{S \leq R}{p(S) dS} = n! \int_{0}^{r_1}p(s_1)\int_{s_1}^{r_2}p(s_2 | s_1)\cdots \int_{s_{n-1}}^{r_n}{p(s_n | s_{n-1}) ds_n ds_{n-1}\cdots ds_1} \label{eq:3}
\end{equation}

where these probabilities are taken under the null hypothesis. It can be shown inductively that the expression in \eqref{eq:3} is equal to $Q(R) = n! V(R)$ if and only if $s_i \sim U(s_{i-1}, 1)\text{ }\forall \text{ }i$.

\begin{theorem}
$\exists \text{ } i | s_i \not\sim U(s_{i-1}, 1) $
\end{theorem}
\begin{proof}
\begin{equation}
    P(s_1 \leq x) = 1 - P(s_1 > x) = 1 - \prod_{i}{P(s_i > x)} = 1 - (1-x)^n \label{eq:4}
\end{equation}

It follows that

\begin{equation}
    p(s_1 = x) = \frac{d}{dx}P(s_1 \leq x) = n (1-x)^{n-1} \label{eq:5}
\end{equation}

which is not a uniform distribution, providing the desired counterexample.
\end{proof}

It follows that $Q(R)$ cannot be used directly as a p-value against the stated null hypothesis. Presumably unaware of this result, one previous approach \cite{stuart} used $Q(R_e)$ as a p-value for the corresponding element $e$, and produced a combined list by sorting (in ascending order) all elements $e$ according to these p-values . A second group \cite{aerts} demonstrated through numerical experiments that $Q(R)$, unlike a valid p-value, did not follow a uniform distribution under the null hypothesis. This group \cite{aerts} went on to measure the empirical distribution of $Q(R)$ under the null hypothesis and fitted it approximately to a $\beta$ distribution for $n \leq 5$ and a $\gamma$ distribution for larger $n$, ultimately using these fitted distributions to convert the joint CDF values to p-values.

\section{Proposed Usage}

While using the same input (several ranked lists of elements) and assuming the same null hypothesis, we impose one additional constraint: that every element present in at least one list is present in all lists. This constraint can be guaranteed as follows: for each list, add all missing elements to the end of the ranking. If $k$ elements are added in this manner to a list of original size $n$, then each of the added elements are assigned a rank of $n + \frac{k+1}{2}$, i.e. the average rank of all such elements had they been added in an arbitrary order.

After this is done, all ranked lists will be identical in size and in the set of elements they contain. Now, for a given element, let $f_i$ be the fraction of lists in which the element appears at rank $i$. It follows that $\sum_{i}{f_i} = 1$. Furthermore, let $g_i = g_{i-1} + f_i$ be the fraction of lists in which the element appears at rank no greater than $i$, with $g_1 = f_1$, and let $r_i = 1 - g_{n-i}$. In line with the intuition underlying the partial ordering we defined in Section 3, we define here a new partial ordering $e \leq h$ if and only if $g_{e,i} \geq g_{h,i}$ $\forall$ $i \iff r_{e, i} = g_{e, n-i} \leq g_{h, n-i} = r_{h, i}$ $\forall$ $i$, where $g_{e, i}$ and $r_{e,i}$ are the $g_i$ and $r_i$, respectively, for element $e$.

Given this, let us redefine $R_e$ to be the list $[r_{e,i}] = [r_i]$ for $1 \leq i \leq n-1$. Using this definition of $r_i$, we have

\begin{equation}
    r_i = 1 - g_{n-i} = 1 - (g_{n-i+1} - f_{n-i+1}) = r_{i-1} + f_{n-i+1} \text{ }\forall \text{ }i > 0 \label{eq:6}
\end{equation}

\begin{theorem}
$r_{i+1} \sim U(r_i, 1)$ $\forall$ $i \geq 0$ under the null hypothesis.
\end{theorem}
\begin{proof}
The recursion in \eqref{eq:6} can be used to show inductively that
\begin{equation}
    r_i = \sum_{k=n-i+1}^{n}{f_k} \label{eq:7}
\end{equation}
Given that, let us consider the distribution of $f_{n-i}$ conditioned on $\{f_k, k > n-i\})$ under the null hypothesis. Since each list is assumed to be randomly permuted, it follows that each element $e$ is equally likely to be in any position for a particular list. Thus, conditioned only on $\{f_k, k > n-i\}$, we have that the probability $p_{n-i}$ of $e$ appearing in position $n-i$ in a particular list, is given by the probability that $e$ hasn't already appeared further down in that list, i.e. $1 - r_i$, multiplied by a uniform probability density 1 over the remaining positions. Since $f_{n-i}$ is, by definition, the expectation of this probability $p_{n-i}$ over many lists, it follows that $f_{n-i} \sim U(0, 1-r_i) \iff r_{i+1} = r_i + f_{n-i} \sim U(r_i, 1)$, as desired.
\end{proof}

Thanks to this result, it follows that for $R_e$ \textit{as we have defined it in this section}, $V(R_e)$ can be used directly as a p-value.

\section{Previous Methods of Computation}

The first approach \cite{stuart} discussed earlier attempted to compute $V(R)$ using

\begin{equation}
    V(R) = \frac{1}{n} \sum_{i = 1}^{n}{(r_i - r_{i-1})V(R_{-i})} \label{eq:8}
\end{equation}

where $R_{-i}$ is defined as $R$ with $r_i$ removed. A straightforward application of dynamic programming to the recursion in \eqref{eq:8} gives a runtime of $O(n!)$ for computing $V(R)$ using this method .

The second approach \cite{aerts} we referenced was able to derive another recursion over an intermediate function $T_k$:

\begin{equation}
    T_k(R) = \sum_{i=1}^{k}{(-1)^{i-1} \frac{T_{k-i}}{i!} (r_{n-k+1})^i} \label{eq:9}
\end{equation}

where $V(R) = T_n(R)$. A straightforward application of dynamic programming to the recursion in \eqref{eq:9} gives a runtime of $O(n^2)$ for computing $V(R)$ using this method.

Both \eqref{eq:8} and \eqref{eq:9} can be derived via manipulations of \eqref{eq:2}, but we have not included these derivations here.

\section{Improved Method of Computation}

For the purposes of this section, let $R = [r_i]$ have length $n$ and let $R_k = [r_i \text{ }\forall\text{ }i \leq k]$ be the subsequence consisting of the first $k$ elements of $R$. We seek to compute $V(R)$ as given by the expression in \eqref{eq:2}. We note that this is simply the $n$-dimensional volume of space that lies on or strictly below the curve defined by $R$. To compute this, we begin by rearranging our integrals, such that the outermost integral is the integral in $s_n$ and the innermost integral is the integral in $s_1$. By construction, we know that $s_i \geq 0$, $s_i \leq r_i$, $s_i \leq s_{i+1}$, and $s_i \leq 1$. Given this, we can write $V(R) = Z_n(R)$ where

\begin{equation}
\label{eq:10}
    Z_i(R) = \begin{cases}
                \bigint_{0}^{\min{(s_2, r_1)}}{ds_1} & \text{if $i=1$} \\
                \bigint_{0}^{\min{(s_{i+1}, r_i)}}{Z_{i-1}(R) ds_i} & \text{if $1 < i \leq n$}
            \end{cases}
\end{equation}

with $s_{n+1} = 1$ set for consistency. Less formally, this can be written as

\begin{equation}
   V(R) = \int_{0}^{r_n}\int_{0}^{\min{(r_{n-1},s_n)}}\cdots \int_{0}^{\min{(r_{1},s_2)}}{ds_1 ds_2\cdots ds_n} \label{eq:11}
\end{equation}

Substituting $R_k$ for $R$ into \eqref{eq:10} and rearranging, we have that

\begin{equation}
   \label{eq:12}
   V(R_k) = \int_{0}^{r_k}{Z_{k-1}(R_k) ds_k} = \int_{0}^{r_k}{V(R_{k-2} + [\min{(r_{k-1},s_k)}])ds_k}
\end{equation}

where list addition refers to concatenation. This can be further rewritten as

\begin{equation}
   \label{eq:13}
   V(R_k) = \int_{0}^{r_{k-1}}{V(R_{k-2} + [s_k])ds_k} + \int_{r_{k-1}}^{r_k}{V(R_{k-1})ds_k}
\end{equation}

Letting $V_k(x) = V(R_{k-1} + [x])$, we can substitute and simplify further to get

\begin{equation}
   \label{eq:14}
   V_k(r_k) = \left(\int_{0}^{r_{k-1}}{V_{k-1}(s_k)ds_k}\right) +(r_k - r_{k-1})V_{k-1}(r_{k-1})
\end{equation}

We note that

\begin{equation}
   \label{eq:15}
    \frac{\partial}{\partial x}V_k(x) = V_{k-1}(r_{k-1})
\end{equation}

i.e. $C = \left(V_k(x) - x V_{k-1}(r_{k-1})\right)$ is constant with respect to x. Consequently, integrating both sides of \eqref{eq:14} with respect to $r_k$ gives

\begin{equation}
   \label{eq:16}
   \int_{0}^{r_k}{V_k(x) dx} = \int_{0}^{r_k}{\left(x V_{k-1}(r_{k-1}) + C\right) dx} = \frac{{(r_k)}^2}{2} V_{k-1}(r_{k-1}) + r_k \left(V_k(r_k) - r_k V_{k-1} (r_{k-1})\right)
\end{equation}

which can be further simplified to

\begin{equation}
   \label{eq:17}
   \int_{0}^{r_k}{V_k(x) dx} = r_k V_k(r_k) - \frac{{(r_k)}^2}{2} V_{k-1}(r_{k-1})
\end{equation}

Substituting $k-1$ for $k$ in \eqref{eq:17}, substituting the resulting expression into \eqref{eq:14}, and simplifying gives
\begin{equation}
   \label{eq:18}
   V_k(r_k) = r_k V_{k-1}(r_{k-1}) - \frac{{(r_{k-1})}^2}{2} V_{k-2}(r_{k-2})
\end{equation}

Since we ultimately wish to find $V_n(r_n)$, \eqref{eq:18} provides us with a convenient recursive formula that we can use to compute $V_n(r_n)$ in $O(n)$ time. This will allow us to feasibly compute the statistic for the proposed usage in Section 4, as the number $n$ of elements being sorted is usually much greater than the number $n$ of lists being aggregated. To use this recursion, we note that the explicit base cases are

\begin{equation}
   \label{eq:19}
   V_1(r_1) = r_1
\end{equation} 

\begin{equation}
   \label{eq:20}
   V_2(r_2) = r_2 r_1 - \frac{(r_1)^2}{2}
\end{equation}

and the implicit base cases are

\begin{equation}
   \label{eq:21}
   V_0 = V([]) = 1
\end{equation}

\section{Implementation}

An implementation of the proposed usage and the improved method of computation can be found on GitHub, under arvindthiagarajan/multimodal-statistics.

\bibliographystyle{plain}  
\bibliography{references}

\nocite{*}

\end{document}